\documentclass[12pt]{article}
\usepackage{filecontents}
\RequirePackage{amsmath, amsfonts, amssymb, amsthm, bm, graphicx, mathtools, enumerate,multirow}

\usepackage{url}
\usepackage[utf8]{inputenc} 
\usepackage[T1]{fontenc}    
\usepackage{url}            
\usepackage{booktabs}       
\usepackage{amsfonts}       
\usepackage{nicefrac}       
\usepackage{microtype}      
\usepackage[shortlabels]{enumitem}

\usepackage{amsthm,amsmath,amssymb,bbm,bm}
\usepackage{cancel}
\usepackage{natbib}
\usepackage{longtable}
\usepackage{multirow}
\usepackage{setspace}
\usepackage{centernot}
\usepackage{array}
\usepackage{algorithmic}
\usepackage[linesnumbered, ruled, vlined]{algorithm2e}
\usepackage{mathrsfs}
\usepackage{dsfont}
\usepackage{relsize}
\usepackage{rotating}
\usepackage{enumitem}
\usepackage{float}
\usepackage{subcaption}
\usepackage{multirow}
\usepackage{wrapfig}
\usepackage{graphicx}
\usepackage{comment}
\usepackage{subcaption}
\usepackage{xcolor}

\DeclareMathOperator*{\argmax}{arg\,max}
\DeclareMathOperator*{\argmin}{arg\,min}

\theoremstyle{definition}
\newtheorem{assumption}{Assumption}

\newtheorem{theorem}{Theorem}[section]
\newtheorem{lemma}{Lemma}[section] 

\newtheorem{corollary}{Corollary}[section]

\newcommand{\QZL}[1]{{\bf {\color{red}\bf [ZL: #1]}}}
\newcommand{\rev}[1]{{\textcolor{black}{#1}}}
\newcommand{\revv}[1]{{\textcolor{black}{#1}}}

\newcommand{\indep}{\rotatebox[origin=c]{90}{$\models$}}
\newcommand{\dep}{\cancel{\rotatebox[origin=c]{90}{$\models$}}}

\newcommand{\EE}{{\mathbb{E}}}

\newcommand{\Vpi}{V^{\pi}}

\newcommand{\qnu}{q^{\nu}}

\newcommand{\jiayi}[1]{{\color{blue}\bf [Jiayi: #1]}}

\newcommand{\super}[1]{\textsc{Super}}
\newcommand{\oonly}[1]{\textsc{Oonly}}
\newcommand{\sonly}[1]{\textsc{Sonly}}
\newcommand{\szonly}[1]{\textsc{SZonly}}
\newcommand{\common}[1]{\textsc{common}}

\newcommand{\ind}{\mbox{$\perp\!\!\!\perp$}}

\newcommand{\pushright}[1]{\ifmeasuring@#1\else\omit\hfill$\displaystyle#1$\fi\ignorespaces}

\newcommand{\abs}[1]{|#1|}

\def\argmax{\operatorname{argmax}}

\newcommand{\samfixed}[1]{}

\def\ds1{{\mathrm{1 \hspace{-2.6pt} I}}}

\def\calA{{\mathcal A}}

\def\calD{{\mathcal D}}

\def\calG{{\mathcal G}}

\def\calM{{\mathcal M}}

\def\calP{{\mathcal P}}
\def\calQ{{\mathcal Q}}
\def\calR{{\mathcal R}}

\def\calS{{\mathcal S}}

\def\calU{{\mathcal U}}
\def\calV{{\mathcal V}}
\def\calW{{\mathcal W}}
\def\calX{{\mathcal X}}

\def\calZ{{\mathcal Z}}

\newcommand{\blind}{1}

\begin{document}

\def\spacingset#1{\renewcommand{\baselinestretch}%
{#1}\small\normalsize} \spacingset{1}


\if1\blind
{
  \title{\bf Blessing from Human-AI Interaction: Super \rev{Policy} Learning in Confounded Environments}
  \author{Jiayi Wang\\
    Department of Mathematical Sciences, University of Texas at Dallas\\
    and \\
    Chengchun Shi\\
    Department of Statistics, London School of Economics and Political Science\\
    and \\
    Zhengling Qi \\
    Department of Decision Sciences, George Washington University
    }
  \maketitle
} \fi

\if0\blind
{
  \bigskip
  \bigskip
  \bigskip
  \begin{center}
    {\LARGE\bf Blessing from Human-AI Interaction: Super  \rev{Policy} Learning in Confounded Environments}
\end{center}
  \medskip
} \fi

\bigskip
\begin{abstract}
  As AI becomes more prevalent throughout society, effective methods of integrating humans and AI systems that leverage their respective strengths and mitigate risk have become an important priority. In this paper, we introduce the paradigm of \textit{super \rev{policy} learning} that takes advantage of Human-AI interaction for data driven sequential decision making. This approach utilizes the observed action, either from AI or humans, as input for achieving a stronger oracle in policy learning for the decision maker (humans or AI). In the decision process with unmeasured confounding, the actions  taken by past agents can offer valuable insights into undisclosed information. By including this information for the policy search in a novel and legitimate manner, the proposed super \rev{policy} learning will yield a \textit{super-policy} that is guaranteed to outperform both the standard optimal policy and the behavior one (e.g., past agents' actions). We call this stronger oracle \textit{a blessing from human-AI interaction}. Furthermore, to address the issue of unmeasured confounding in finding super-policies using the batch data, a number of nonparametric and causal identifications are established \rev{under the framework of proximal causal inference.} Building upon on these novel identification results, we develop several super-policy learning algorithms and 
  systematically study their theoretical properties such as finite-sample regret guarantee. Finally, we illustrate the effectiveness of our proposal through extensive simulations and real-world applications.
\end{abstract}

\noindent%

\spacingset{1.7} 

\abovedisplayskip=3pt
\abovedisplayshortskip=3pt
\belowdisplayskip=3pt
\belowdisplayshortskip=3pt

\section{Introduction}
In recent years, AI has become increasingly important in solving complex tasks throughout society. While in many applications it is crucial to have fully autonomous systems that involve little or even no human interaction, in high-stake domains ranging from autonomous driving \cite{levine2020offline}, medical studies \citep{kosorok2019precision} to algorithmic trading \citep{liu2022finrl}, integrating AI systems and human knowledge is arguably the most effective for better decision making. Motivated by this, we study offline reinforcement learning under unmeasured confounding, where human-AI interaction can be naturally incorporated for better decision making.

Offline reinforcement learning (RL) aims to find a sequence of optimal policies by leveraging the batch data collected from past agents \citep{sutton2018reinforcement,levine2020offline}. In contrast with online learning, where agents can interact with environment through trial and error, offline RL must rely entirely on the pre-collected observational or experimental data and the agents have no control of the data generating process. More importantly, the possible existence of unobserved variables/confounders in the offline setting posits a significant challenge that may hinder an agent from learning an optimal policy. Despite these challenges, we observe that due to the unmeasured confounding, the behavior policy used to generate the data may reveal additional valuable information that is not recorded in the observed variables. In this paper, we propose a paradigm of super \rev{policy} learning by correctly incorporating the observed actions in the offline data for policy search, which is guaranteed to outperform the existing decision making methods. The proposed approach offers a unique opportunity for the human-AI interaction that leads to a better decision making.



\subsection{Motivating Examples}

\textbf{Machine in the Human Loop.} 
Since the middle of last century, the emergence of AI has had a profound impact on business operations,  particularly in the field of financial trading. 
AI algorithms are heavily used to discover market patterns and recommend trading strategies for maximizing profits \citep{johnson2010algorithmic, chan2021quantitative}. Compared with human traders, these algorithms are highly effective in analyzing large amounts of observational data, including historical and real-time financial, social/news and economic data, to make complex decisions. 
\rev{While AI is extremely powerful, 
there are also risks associated with relying solely on AI for decision-making \citep{saltapidas2018financial} and  
AI results can be quite unstable due to such a large size of features in its training process.  A slight change of a variable or a different choice of machine learning algorithms can result in a significant impact on the performance \citep{treleaven2013algorithmic}.  
On the other hand, human agents still play a fundamental role because they possess a unique ability to understand context, recognize patterns and make judgments based on their experience and knowledge. However, traditional decision making strategies from human agents may not attain optimality, as they lack the ability to extract all useful information manually. By integrating AI recommendations, human agents gain the capacity to assimilate machine-provided insights gleaned from the analysis of extensive offline data generated by AI algorithms, enabling them to make more informed and better decisions.}

\medskip

\noindent \textbf{Human in the Machine loop.} 
In many other applications, there is a common belief that human decision-makers have access to important information when taking an action \citep{kleinberg2018human}. For example, in the urgent care, clinicians leverage
visual observations or communications with patients to recommend treatments, where such
unstructured information is hard to quantify and often not recorded \citep{mcdonald1996medical}. 
In autonomous
driving, measurements collected by sensors are often noisy, causing partial observability that prevents autonomous agents from learning optimal actions. 
Hence it is commonly advocated that self-driving cars should be overseen by human drivers to serve as important safeguards against unseen dangers
\citep{nyholm2020automated}. 
Take the deep brain stimulation \citep[DBS][]{lozano2019deep} as a concrete example. Due to recent advances in DBS technology, it becomes feasible to instantly collect electroencephalogram data, based on which we are able to provide adaptive stimulation to specific regions in the brain so as to treat patients with neurological disorders including Parkinson's disease, essential tremor, etc. In this application, the patient is allowed to determine the behavior policy (e.g., when to turn on/off the stimulation, for how long, etc) based on the information only known to herself (e.g., how she feels), generating batch data with unmeasured confounders. Even though the human's decision may not be the optimal to herself due to her inability to objectively analyze the whole body environment,  it reflects her mental and physical information  that is difficult to record in the data. On the other hand, a machine decision maker can make the full use of the electroencephalogram data. By including patient's action in the learning process, the machine's recommendations for the treatment can be potentially improved.

\medskip

\noindent \textbf{Summary}. To summarize, in many applications, 
intermediate decisions given by human or machine (e.g., either an AI algorithm in stock trading or a patient during DBS therapy), naturally provide additional information  for achieving a stronger oracle in policy learning, compared with methods only based on observed covariates information.
This is indeed what we call ``a blessing from Human-AI Interaction" in the data-driven decision making. 

\subsection{Contribution: Super-policy Learning} 
Our contributions can be summarized in four-fold.   First, we introduce a novel decision-making paradigm in the confounded environment called super RL. Compared with the standard RL, super RL additionally takes the behavior agent's recommendations 
as input for learning an optimal policy, which is guaranteed to achieve a stronger oracle. In the confounded environment, 
super RL can embrace the blessing from behavior policies given by either AI or human.  In other words, it leverages the expertise of behavior agents in discovering unobserved information 
for enhanced policy learning for the current decision maker. The resulting policy, which we call \textit{super-policy}, is guaranteed to outperform the standard optimal one in the existing literature. 
To implement the proposed super-policy in the future, we require the behavior agent to recommend an action at each time, a common practice in certain applications, as discussed in our motivating examples. 
Second, to address the challenge of unmeasured confounding that hinders us from learning a super-policy using the offline data, we establish several novel non-parametric and causal identification results in various confounded environments for learning super policies. It is significantly challenging for identifying the super-policy in the sequential setting as we need to include two sets of past actions for policy search, i.e., those generated by the behavior policy and by the super-policy.  
Moreover, based on these identification results, we develop several super RL algorithms and derive the corresponding finite-sample regret guarantees. Finally, numerous simulation studies and real-world applications are conducted to illustrate the superior performance of our methods.


\subsection{Related Work}


There are a growing body of literature delving into the realm of human-AI interaction within the context of reinforcement learning. Various perspectives have been taken to incorporate human's knowledge in the learning process. 
Among these, the most common strategy is known as ``reward shaping". This involves tailoring the reward function of reinforcement learning through human feedback, with the aim of enhancing the agent's behavior \citep[e.g.][]{tenorio2010dynamic,warnell2018deep}. 
Another line of research uses human knowledge to adjust the policy. For instance, \cite{griffith2013policy} introduces a Bayesian approach that employs human feedback as a policy label to refine policy shaping; \cite{fachantidis2017learning} utilizes human knowledge to guide the exploration process of agents, while \cite{brys2015reinforcement} combines the value function generated by the agent with the one derived from human feedback to amplify the learning process; Moreover, \cite{saunders2017trial} puts forth the concept of safe RL through human intervention, wherein human intervention serves to override unfavorable actions recommended by the intelligent agent.
Different from the aforementioned approaches, our proposed ``super RL" takes a unique perspective.  We aim to leverage either the human or AI expertise in the previously collected data for helping the other side make better decisions. %

Our work is also closely related to off-policy evaluation (OPE) and learning under unmeasured confounding
in the sequential decision making problem \citep[Section 6.2]{uehara2022review}. Specifically, \cite{zhang2016markov} 
introduced the causal RL framework and the confounded Markov decision process (MDP) 
with memoryless unmeasured confounding, under which the Markov property holds in the observed data. Along this direction, many OPE and learning methods are 
proposed using instrumental or mediator variables \citep{chen2021estimating,liao2021instrumental,li2021causal,wang2021provably,shi2022off,Fu2022,yu2022strategic,xu2023instrumental}. 
In addition, partial identification bounds for the off-policy's value have been established based on sensitivity analysis \citep{namkoong2020off,kallus2020confounding,bruns2021model}. Another streamline of research focuses on general confounded POMDP models that allow time-varying unmeasured confounders to affect future rewards and transitions. 
Several point identification results were established 
\citep{tennenholtz2020off,bennett2021proximal,nair2021spectral,shi2021minimax,ying2021proximal,Miao2022} in this setting. However, none of the aforementioned works study policy learning with the help of the behavior agent's expertise, i.e., taking recommended action in the observed data for decision making.
Different from these works, we tackle the policy learning problem from a unique perspective and propose a novel super RL framework by leveraging the behavior agent's  expertise in discovering certain unobserved information to further improve decision making. 
\rev{Our identification results are developed by extending the idea of proximal causal inference \citep{tchetgen2020introduction}.}
We also rigorously establish the super-optimality of the proposed super-policy over the standard optimal policy and the behavior policy.
Our paper is also related to a line of works on policy learning and evaluation with partial observability using spectral decomposition and predictive state representation related methods \citep[see e.g.,][]{littman2001predictive,song2010hilbert,boots2011closing,hsu2012spectral,singh2012predictive,anandkumar2014tensor,jin2020sample,cai2022sample,uehara2022provably, lu2022pessimism,uehara2022future}. Nonetheless, these methods require the no-unmeasured-confounders assumption.

Finally, our proposal is motivated by the work of \cite{stensrud2022optimal} 
that introduced the concept of superoptimal treatment regime in contextual bandits. 
\rev{They used an instrumental variable approach for discovering such regime.   However, their method can only be applied in a restrictive single-stage decision making setting with binary actions. 
 Different from their approach for the identification, we adopt the framework of proximal causal inference. 
More importantly, our super-RL framework is generally applicable to both confounded contextual bandits and sequential decision making and can allow arbitrarily many but finite actions. }

\section{Super \rev{Policy} Learning
}

\subsection{Super Policy for Contextual Bandits}
\label{sec:contextual}
In this section,  we introduce the idea of super RL 
in the confounded contextual bandits (e.g., single-stage decision making under endogeneity). 
Consider a random tuple $(S, U, A, \{R(a)\}_{a \in \calA})$, where $S$ and $U$ denote the observed and unobserved features respectively. Denote their corresponding spaces as $\calS$ and $\calU$. 
The random variable $A$ denotes the taken action whose space is given by a finite set $\calA$, and $\{R(a)\}_{a \in \calA}$ denotes a set of the potential/counterfactual rewards, where $R(a)$ represents the reward that the agent would receive had action $A = a$ been taken. Assuming consistency in the causal inference literature \citep[see e.g.,][]{rubin2005causal}, the observed reward in the data, denoted by $R$, can then be written as $R = \sum_{a \in \calA}R(a)\mathbb{I}(A = a)$. 

Consider a policy $\pi: \calS \rightarrow \calP(\calA)$ as a function mapping from the observed feature $S$ into $\calP(\calA)$, a class of 
all probability distributions over $\calA$. 
In particular, $\pi(a|s)$ refers to the probability of choosing an action $a$ given that $S=s$. In the batch setting, we are 
given i.i.d. copies of $(S, A, R)$, where the action $A$ is generated by some behavior policy $\pi^b: \calS \times \calU \rightarrow \calP(\calA)$ that may depend on both observed and unobserved features. 
Nearly all existing solutions 
focus on finding an optimal policy $\pi^\ast$ such that
\begin{align}\label{def: regular optimal 2}
    \pi^\ast(a^\ast| s) = 1 \quad \text{if} \quad  a^\ast = \argmax_{a \in \calA}\EE\left[R(a) | S = s\right] \quad \forall s \in \calS.
\end{align}
Here we assume the uniqueness of the maximization in \eqref{def: regular optimal 2} for every $s \in \calS$. Since $U$ may confound the causal relationship of the action-reward pair in the observational data, direct implementing standard policy learning methods  will produce a biased estimator of $\pi^{\ast}$ \citep{qi2023proximal,yu2024two}. Besides, in the presence of latent confounders, 
there is \textit{no} guarantee that the standard optimal policy $\pi^\ast$ outperforms the behavior policy $\pi^b$ because $\pi^b$ depends on the unobserved information.

As discussed earlier, we take a unique perspective on this problem and aim to find a better policy beyond the standard optimal policy $\pi^\ast$ and the behavior policy $\pi^b$. The key idea is to treat the action $A$ generated by the behavior policy $\pi^b$ as an additional feature to $S$ for seeking a stronger oracle because the observed action depends on the unobserved feature $U$ and may have more information for making a better decision. 
Specifically, we 
define 
a \textit{super-policy} $\nu^\ast$ in a larger policy class 
$\Omega = \{\nu: \calS \times \calA \rightarrow \calP(\calA)\}$ such that
\begin{align}\label{def: super policy}
\nu^\ast(a^\ast | s, a') = 1 \quad \text{if} \quad  a^\ast = \argmax_{a \in \calA}\EE\left[R(a) | S = s, A = a'\right] \quad \forall (s, a') \in \calS \times \calA.
\end{align}
Here, $a'$ corresponds to the action recommended by the behavior policy, which may differ from $a^*$, the action recommended by the proposed super-policy. However, since the behavior policy $\pi^b\in \Omega$, the proposed super-policy 
is always better than $\pi^b$.
Additionally, notice that  
$$
    \EE\left[R(a) | S = s, A = a'\right] \neq \EE\left[R(a) | S = s\right],
$$
in general, because the unobserved feature $U$ will affect the distribution of the counterfactual rewards $\{R(a)\}_{a \in \calA}$ under different interventions of $A$. Hence $\nu^\ast$ can be different from $\pi^\ast$. 
See Figure \ref{fig:contextual_bandits_policies} for an illustration of the standard optimal policy $\pi^*$ and the proposed super policy $\nu^*$. 
Specifically, let $\calV(\nu)$ be the value (i.e., expected reward) under the intervention of a generic policy $\nu \in \Omega$, i.e., 
$$
\calV(\nu) = \sum_{a \in \calA}\EE[R(a)\nu(a\mid S, A)].
$$ We have the following lemma that demonstrates the \textit{super-optimality} of $\nu^\ast$.  over both $\pi^\ast$ and $\pi^b$. 
\begin{lemma}[Super-Optimality]\label{lm: superoptimality}
 $\calV(\nu^\ast) \geq \max\{ \calV(\pi^b), \calV(\pi^\ast)\}$.
\end{lemma}

\begin{figure}[h]
\centering
    \vspace*{-80pt}
    \includegraphics[width = 0.75\textwidth]{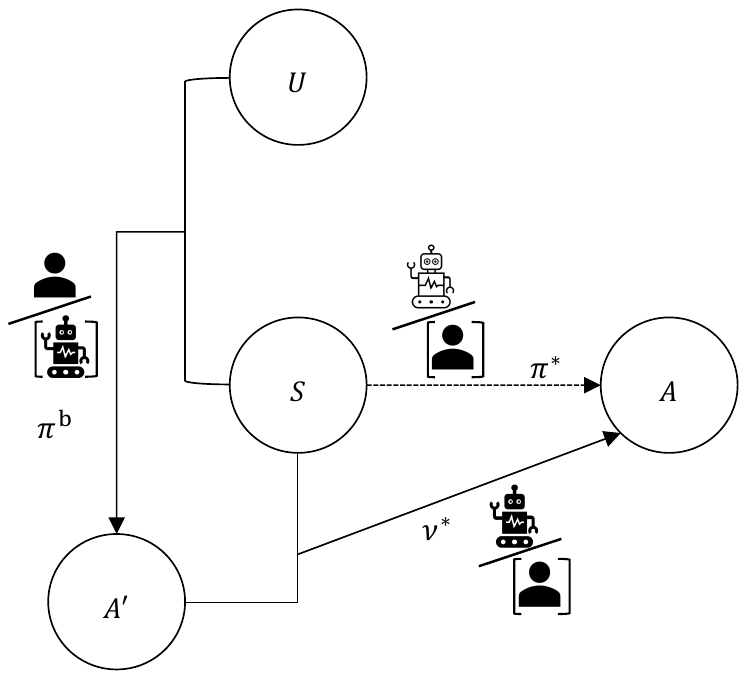}
    \vspace*{-70pt}
    \caption{Graphical models for the data generating process under a behavior policy $\pi^b$, a standard optimal policy $\pi^*$ (dashed line) and a super policy $\nu^*$. }
    \label{fig:contextual_bandits_policies}
\end{figure}

\subsection{An Illustrative Example on Human-AI Interaction}
\noindent
Lemma \ref{lm: superoptimality} ensures the advantage of leveraging Human-AI interaction in the decision making, so-called ``a blessing from Human-AI interaction". 
For instance, one can interpret the behavior policy $\pi^b$ as given by the AI system, which is capable of providing decisions based on massive information. As a human decision maker, despite the limitation to access all data information, she can make a better decision $\nu^\ast$ based on the recommendation given by the AI system. One can also interpret $\pi^b$ as the behavior policy given by the human agent which involves unique human insights. The super-policy $\nu^*$ learned by the machine utilizes such information and are guaranteed to outperform the human agents and the common policy that does not rely on human recommendations.  
In Section \ref{sec:finite horizon}, we extend our framework to the setting where human and AI are iteratively interacting and making better decisions via finding the super-policy.

To further understand the appealing property of the proposed super-policy, consider the following toy example. Assume $S$ and $U$ independently follow a Bernoulli distribution with a success probability $0.5$.  Suppose the action is binary ($\calA=\{0, 1\}$) and the behavior policy satisfies $\mathbb{P}(A=1|S,U=1)=\mathbb{P}(A=0|S,U=0)=1-\varepsilon$ for some $0\le \varepsilon\le 1$. Let $R=8(A-0.5)(S-0.2)(U-0.3)$. In this example, the parameter $\varepsilon$ measures the degree of unmeasured confounding. When $\varepsilon=0.5$, the behavior policy does not depend on $U$ and the no-unmeasured-confounding assumption is satisfied. Otherwise, this condition is violated. In particular, when $\varepsilon=0$ or $1$, we can fully recover the latent confounder based on the recommended action. 
Table \ref{tab:toy_example} summarizes the policy values of $\pi^b$, $\pi^\ast$ and $\nu^\ast$ under different $\varepsilon$, in which the super-optimality clearly holds. 
\begin{table}[ht!]
\caption{Policy values under different choices of $\varepsilon$ in the toy example. In general, $\calV(\pi_b) = 0.6 - 1.2\varepsilon$, $\calV(\pi^\ast) = 0.4$, $ \calV(\nu^\ast) = |0.7 - \varepsilon| + |\varepsilon - 0.3|$. Bold values are the largest  under different settings.  }
\label{tab:toy_example}
\begin{center}
    \vspace*{-20pt}
\begin{tabular}{| l|lll |}
\hline
Policy Value &\bf $\calV(\pi_b)$  &\bf $\calV(\pi^\ast)$ &$\calV(\nu^\ast)$\\ \hline 
$\varepsilon = 0.5$  &    0.0  & \textbf{0.4}   & \textbf{0.4} \\
\hline
$\varepsilon = 0$     &    0.6   &  0.4   &   \textbf{1.0}   \\\hline
$\varepsilon = 1$      &   -0.6   &0.4 &  \textbf{1.0}\\\hline
\end{tabular}
\vspace*{-25pt}
\end{center}
\end{table}

\subsection{When is the super-optimality strict?}
As seen from Table \ref{tab:toy_example}, when $\varepsilon = 0$, the super-policy has the same performance as the standard optimal one $\pi^\ast$. This is due to the fact that the behavior policy $\pi^b$ does not provide additional information. To further understand when the strict improvement of $\nu^\ast$ over $\pi^b$ and $\pi^\ast$ happens, consider a binary-action setting with $\calA = \{0, 1\}$, where $1$ denotes the new treatment group and $0$ denotes the standard control. Define the conditional average treatment effect on the treated (CATT) and on the control (CATC) respectively as 
    \begin{align*}
        \mathrm{CATT}(s) = \EE\{R(1) - R(0) \mid S = s, A = 1 \},\\
        \mathrm{CATC}(s) = \EE\{R(1) - R(0) \mid S = s, A = 0 \}.
    \end{align*}
    given any $s \in \calS$. Then we have the following lemma, which explicitly characterizes the super-optimality of $\nu^\ast$ over $\pi^\ast$ and $\pi^b$.

    \begin{lemma}
\label{lem:improvement}
    The following three results hold.\\
    (i) $\calV(\nu^*) > \calV(\pi^*)$ if and only if 
    $
    \Pr\left(\left\{ 0 < \pi^b(1 | S ) < 1\right\} \cap \left\{\mathrm{CATT}(S) \times \mathrm{CATC}(S)<0\right\}\right) >0;
    $
    (ii) $\calV(\nu^*) > \calV(\pi^b)$ if and only if
    $
    \Pr\left(\left\{ \mathrm{CATT}(S) < 0\right\} \cup \left\{\mathrm{CATC}(S)>0\right\}\right) >0;
    $\\
    and (iii) $\calV(\nu^*) > \max\{\calV(\pi^*), \calV(\pi^b)\}$ if and only if 
    $$
    \Pr\left(\left\{ 0 < \pi^b(1 | S ) < 1\right\} \cap \left\{ \mathrm{CATT}(S) < 0\right\} \cap \left\{\mathrm{CATC}(S)>0\right\}\right) >0;
    $$
\end{lemma}
\noindent
Result (i) of Lemma \ref{lem:improvement} indicates that by treating $A$ as an additional feature, as long as $A$ is informative for achieving a better expected reward for some feature $S$, we have the strict improvement of $\nu^\ast$ over $\pi^\ast$. Meanwhile, Result (ii) of Lemma \ref{lem:improvement} implies that as long as the alternative action is better than the one recommended by $\pi^b$ for some feature $S$, strict improvement of $\nu^\ast$ over $\pi^b$ is guaranteed. Clearly, Result (iii) ensures the strict super-optimality when the previous two scenarios happen simultaneously.

\subsection{Super RL for Sequential Decision Making} 
\label{sec:finite horizon}

In this section, we introduce the super-policy for confounded sequential decision making and demonstrate its super-optimality. 
For any generic sequence $\{X_t\}_{t\geq 1}$, its realization $\{x_t\}_{t\geq1}$ and its spaces $\{\calX_t\}_{t\geq1}$, we
denote $X_{1:t} = (X_1,\dots, X_t)$, $x_{1:t} = (x_1, \dots, x_t)$ and $\mathcal{X}_{1:t} = \prod_{t'=1}^t \calX_{t'}$.

Consider an episodic and confounded stochastic process denoted by $\calM = ( T,  \mathcal{O}, \calU, \calA, 
 \mathcal{P}, \calR)$, where the integer $T$ 
is the total length of horizon, $\mathcal{O} =\{\mathcal{O}_t\}_{t=1}^T$ and $\calU = \{\calU\}_{t=1}^T$ denote the spaces of observed and unobserved features respectively, 
$\calA = \{\calA_t\}_{t=1}^T$ denotes 
the action spaces across $T$ decision points,  $\mathcal{P} = \{\mathcal{P}_t\}_{t=1}^T$ where each $\calP_t$ denotes transition kernel from $\prod_{t'=1}^{t} (\mathcal{O}_{t'} \times \calU_{t'} \times \calA_{t'}) \rightarrow \mathcal{O}_{t+1} \times \calU_{t+1}$ at time $t$, and $\calR$ denotes the set of rewards. 
The random process following $\calM$ can be summarized as $\{O_t, U_t, A_t, R_t\}_{t=1}^T$, where $O_t$ and $U_t$ correspond to the observed and latent features at time $t$, $A_t$ and $R_t$ denote the action and the reward at time $t$. We assume that $O_t$ is some noisy mapping of $U_{1:t}$ and satisfies that $O_t \ind (O_{1:t-1}, A_{1:t}) | U_{1:t}$ for every $1\leq t\leq T$. 
For simplicity, we assume the action space is discrete and all rewards are uniformly bounded, i.e., $\abs{R_t} \leq R_{\max}$.   In the offline setting, we assume the observed action $A_t$ in the batch data is generated by some behavior policy $\pi^b_t:   \calU_{1:t}  \times \calA_{1:t-1} \rightarrow \calP(\calA_t)$ for $1 \leq t \leq T$ and let $\pi^b = \{\pi^b_t\}_{t=1}^T$;  Lastly, we denote the reward function as 
\rev{$r_t: \mathcal{U}_{1:t} \times \mathcal{A}_{1:t} \times \mathcal{O}_t \rightarrow \mathbb{R}$, i.e., $r_t = \EE[R_t | U_{1:t} = \bullet, A_{1:t} = \bullet, O_t = \bullet]$.}


Given the decision process $\calM$ generated by the behavior policy,  the objective of an agent is to learn an (in-class) optimal  policy 
that can maximize the expected cumulative rewards. Nearly all existing 
works are focused on policies defined as a sequence of functions mapping from the past history (excluding the actions produced by the behavior agent) to a probability mass function over the action space $\calA_t$. Specifically, let $\Pi\equiv\{\pi = \{\pi_t \}_{t=1}^T \, \mid \, \pi_t: \mathcal{O}_{1:t} \times \calA_{1:t-1} \rightarrow \calP(\mathcal{A}_t)\}$. Then for any $\pi \in \Pi$, one can use the policy value to evaluate its performance, which is defined as
\begin{align}\label{def: integrated value fun}
	\calV(\pi) = \EE[\Vpi_1(O_1, U_1)],
\end{align}
where we use $\EE$ to denote the expectation with respect to the initial data distribution. Here the value function $\Vpi_t$ is defined as for every $(o_{1:t}, u_{1:t}) \in \mathcal{O}_{1:t} \times  \calU_{1:t}$
\begin{align}
    \label{eqn: Q-function}
\textstyle	\Vpi_t(o_{1:t}, u_{1:t}) = \EE^\pi\left[\sum_{t' = t}^{T}R_{t'} | O_{1:t} = o_{1:t}, U_{1:t} = u_{1:t}\right],
\end{align}
where $\EE^\pi$ denotes the expectation with respect to the distribution whose action at each time $t$ follows $\pi_t$. Since $\{U_t\}_{t\geq1}$ is not observed,
 previous works such as \cite{lu2022pessimism} focus on finding $\pi^\ast$ such that 
 $$
 \pi^\ast \in \argmax_{\pi \in \Pi}\calV(\pi).
 $$
 Similar to the contextual bandit setting, $\pi^\ast$ is not guaranteed outperforming $\pi^b$.
 Motivated by the discussions in Section \ref{sec:contextual}, we consider a much larger policy class 
 $$
 \Omega \equiv \{\nu = \{\nu_t \}_{t=1}^T \, \mid \, \nu_t: \mathcal{O}_{1:t} \times \calA_{1:t-1} \times \calA_{1:t} \rightarrow \calP(\mathcal{A}_t)\},
 $$ 
 where $\calA_{1:t-1}$ represents the past actions taken by the policy $\nu$ up to $(t-1)$ decision points and $\calA_{1:t}$ represents the actions generated by the behavior policy up to $t$ decision points. The policy class $\Omega$ reflects the iterative interaction between human and AI because either $\nu$ or $\pi^b$ can be regarded as human or AI. Therefore, we 
 propose to learn a super policy $\nu^\ast$ such that
 $$
 \nu^\ast \in \argmax_{\nu \in \Omega}\calV(\nu),
 $$
 which leverages human or machine expertise for enhanced decision making that maximizes $\calV(\nu)$. By sequentially integrating historical and current actions of the behavioral agent into the decision-making process, the super policy leverage the iterative human-AI interaction, as illustrated in Figure \ref{fig:DTR_policy}.

\begin{figure}[h]
\centering
    \vspace*{-60pt}
    \includegraphics[width = 0.75\textwidth]{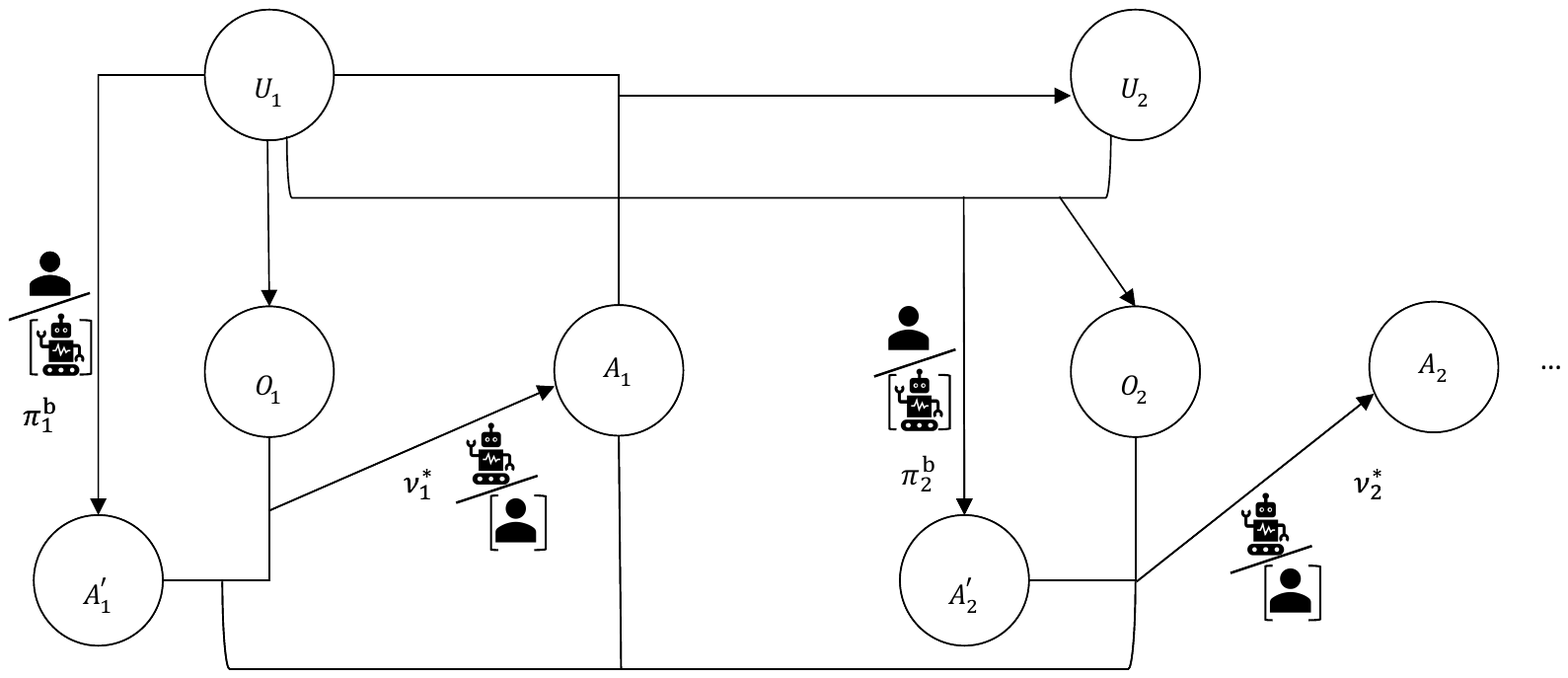}
    \vspace*{-60pt}
    \caption{Graphical model for the sequential decision making under super policy $\nu^*$ by leveraging the offline data generated by the behavior policy  with $T=2$. Here  $\pi^b_1$ depends on $U_1$, $\pi^b_2$ depends on $U_1, U_2, A_1$; $\nu^*_1$ depends on $O_1, A'_1$, $\nu^*_2$ depends on $O_1, O_2, A_1, A'_1, A_2'$.  }
    \label{fig:DTR_policy}
\end{figure}

Similar as before, 
since the super-policy additionally uses the recommendation generated by the behavior policy that depends on the unobserved information, we expect
the super-policy $\nu^\ast$ superior to both $\pi^\ast$ and $\pi^b$, which is shown below.
\begin{theorem}[Super-Optimality]\label{prop:sup_in_finite}
   $\calV(\nu^*)\ge \max\{\calV(\pi^*), \calV(\pi^b)\}$.
\end{theorem}
Given the appealing property of the super policy, in the following section, we discuss how to identify it using the offline data.

\section{Causal Identification for Super Policies} 
\label{sec:identification}
Despite the appealing property of the proposed super policy, it is generally impossible to learn $\nu^\ast$ without any further assumptions, since for example, in the contextual bandit setting, the counterfactual effect $\EE\left[R(a) | S = s, A = a'\right]$ is not identifiable from the observed data due to unmeasured confounding.  In this section,  we extend the idea of proximal causal inference \citep{tchetgen2020introduction} to address the challenge posed by unmeasured confounding, and develop several nonparametric identification results for super policies within the contexts of the contextual bandit and the sequential setting.  \rev{The existing identification results in the proximal causal inference are only designed for finding standard optimal policies that depend on the observed state variables. The proposed super-policy additionally depends on the action $A$ generated by the behavior policy. Its identification is challenging since $A$ always equals the selected action in the observational data, while we need to learn the value where the selected action differs from $A$ for super-policy learning.  Extensions to the sequential setting are even more challenging to develop,  as we aim to include two sets of past actions for policy search, i.e., those generated by the behavior policy and by the super-policy. }

\subsection{\revv{Identification in Contextual Bandits}}
\label{sec:identification_contextual_bandits}
In the bandit setting,  similar to \cite{tchetgen2020introduction}, we assume the existence of certain action and reward proxies $Z \in \calZ$ and $W \in \calW$ in 
addition to $(S, A, R)$.  
\rev{With the help of these two proxies, we introduce the bridge function $q$ and verify that it can be used to identify the value of any super policy $\nu \in \Omega$. Then we provide identification for $q$ function using the observed data.  }
The two proxies are required to satisfy the following assumptions:
\begin{assumption}\label{ass: proximal}
(a) $R \indep Z | (S, U , A)$; (b) $W \indep (Z, A) | (S, U)$, $W \dep U | S $; (c) $R(a) \indep A | (S, U)$ for $a \in \calA$; (d) There exists a bridge function $q: \calW \times \calA \times \calS \rightarrow \mathbb{R}$ such that
    \begin{align}\label{def: outcome bridge}
        \EE\left[q(W, a, S) | U, S, A = a \right] = \EE\left[R | U, S, A = a \right].
    \end{align}
\end{assumption}
Assumptions \ref{ass: proximal}(a)-(b) are standard in proximal causal inference \citep{miao2018confounding}. Assumptions \ref{ass: proximal}(c), which is called latent unconfoundedness, is mild as we allow $U$ to be unobserved. \revv{There are many examples of proxy variables in practice. In DBS, for instance, the negative severity of the disease can serve as the reward, while the underlying mental state acts as the unmeasured confounder.  For the action proxy variable $Z$,  a suitable choice could be the patient's historical medical records.  As for the reward proxy variable $W$,  a practical example could be pre-treatment evaluations derived from questionnaires, such as patients' self-reported anxiety ratings. These ratings reflect the underlying mental state and satisfy the necessary conditions for a reward proxy variable.
} {In {Section S1 }, we provide more detailed discussion for $W$ and $Z$.} The last assumption can be satisfied when some completeness and regularity conditions hold. See \cite{miao2018identifying} and also Lemma \ref{lm: solve q contextual} below for more details.
The following identification result allows us to learn the super-policy $\nu^\ast$ from the observed data. We remark that this result is new and different from the standard proximal causal inference. \rev{In addition to the Assumptions for proxy variables. We also assume the following basic assumptions in causal inference.}
\begin{assumption}
    \label{ass:consistency}
\rev{(a) Consistency: $R = R(A)$. (b) Positivity: $P(A= a \mid S, Z) > 0$ for any $a \in \calA$ almost surely.}
\end{assumption}
\begin{lemma}
\label{lm: Ra_identification.}
Under Assumption \ref{ass: proximal} and \ref{ass:consistency}, 
we have $\EE\left[R(a) | S = s, A = a'\right] = \EE[q(W, a, S) | S = s, A = a']$ for every $(s, a, a')$. Then for any $\nu \in \Omega$,
$$
    \calV(\nu) = \EE\left[\sum_{a \in \calA}q(W, a, S)\nu(a| S, A)\right].
    $$ 
\end{lemma}
In practice, one may want to include as many confounders in the policy as possible 
to achieve the largest super-optimality. Hence under this proximal causal inference framework, with some abuse of notation, we further extend the policy class to $\Omega = \{\nu: \calS \times {\calZ} \times \calA \rightarrow \calP(\calA) \}$ and consider the corresponding super-policy $\nu^\ast$ as 
\begin{align}\label{def: super policy contextual}
\nu^\ast(a^\ast | s, {z}, a') = 1 \quad \text{if} \quad  a^\ast = \argmax_{a \in \calA}\EE\left[R(a) | S = s, {Z} = {z}, A = a'\right],
\end{align}
In applications where the action proxy is no longer available in future decision making, \eqref{def: super policy contextual} is reduced to \eqref{def: super policy}. We remark that different from $Z$, $W$ is usually obtained after intervention, which should not be included in the super-policy. The next corollary allows us to identify $\nu^{\ast}$.
\begin{corollary}\label{cor: sup-policy on Z}
Under Assumption \ref{ass: proximal} and \ref{ass:consistency}, the policy value under a given $\nu\in\Omega$ is given by $
    \calV(\nu) = \EE\left[\sum_{a \in \calA}q(W, a, S)\nu(a| S, A, {Z})\right]$. 
In addition, the optimal policy $\nu^\ast$ is given by
\begin{align}\label{def: super policy with Z}
\nu^\ast(a^\ast | s, {z}, a') = 1 \quad \text{if} \quad  a^\ast = \argmax_{a \in \calA}\EE\left[q(W, a, S) | S = s, {Z} = {z}, A = a'\right].
\end{align}
\end{corollary}
It can be seen from Corollary \ref{cor: sup-policy on Z} that to identify the super-policy, it remains to estimate the bridge function $q$ defined in Assumption \ref{ass: proximal}(d). One can impose the following completeness condition to consistently estimate it. 
\begin{assumption}[Completeness]\label{ass: completeness contextual}
    (a) For any squared-integrable function $g$ and for any $(s, a) \in \calS \times \calA$, $\EE[g(U) | Z, S = s, A = a] = 0$ almost surely if and only if $g(U) = 0$ almost surely. \\
    (b) For any squared-integrable function $g$ and for any $(s, a) \in \calS \times \calA$, $\EE[g(Z) | W, S = s, A = a] = 0$ almost surely if and only if $g(Z) = 0$ almost surely.
\end{assumption}

Completeness is a technical assumption commonly adopted in value identification problems. It can be satisfied by a wide range of models and distributions in statistics and econometrics \citep[e.g.][]{d2011completeness, chen2014local,tchetgen2020introduction}. For (a), it indicates that $Z$ should have sufficient variability compared to the variability of $U$, which helps to make \eqref{eqn: linear integral equation contextual} in the following hold when replacing  $Z$  with $U$. The condition (b) is mainly proposed for ensuring the existence of solution for \eqref{eqn: linear integral equation contextual}. 

\begin{lemma}
\label{lm: solve q contextual}
 Under Assumptions \ref{ass: proximal}-\ref{ass: completeness contextual} and some regularity conditions (see Assumption S1 
 in {Section S2}), solving the following linear integral equation
 \begin{align}\label{eqn: linear integral equation contextual}
     \EE\left[q(W, a, S) | Z, S, A = a \right] = \EE\left[R | Z, S, A = a \right],
 \end{align}
 for every $a \in \calA$ with respect to $q$ gives a valid bridge function that satisfies Assumption \ref{ass: proximal}(d).
\end{lemma}

Built upon Corollary \ref{cor: sup-policy on Z} and Lemma \ref{lm: solve q contextual}, 
we can estimate the bridge function $q$ using the observed data and therefore obtain an estimation of super policy $\nu^*$. See Section \ref{sec:est_contextual}  for more details and the proposed algorithm. 


\subsection{\revv{Identification in Sequential Decision Making}}
\label{sec:identification_DTR}
Next, we discuss how to identify $\nu^\ast$ in the sequential setting. 
In the following,  
{we \revv{introduce} two approaches to identify policy values with the help of some proxy variables. Both are motivated by the recent development of identifying off-policy value in the confounded POMDPs \citep{bennett2021proximal,tennenholtz2020off,shi2021minimax,Miao2022}.
One approach requires an additional proxy variable which is independent of the past action at each decision point (Section \ref{sec:fh1}) 
and an efficient learning algorithm can be developed under the memoryless assumption (Section \ref{sec:est_fh}). The other approach, which does not need the proxy variables, builds on the current data generating process 
but with the loss of an efficient learning algorithm. Due to the space limitation, we present the second approach in {Section  S3}. 
}
 

\subsubsection{\revv{Identification via Q-bridge Functions}}
\label{sec:fh1}
In this subsection, we develop the framework for the identification of super-policies in the confounded sequential decision making settings via $Q$-bridge functions. 
Note that in our policy class $\Omega$,  the policy $\nu_t$ depends on two sets of actions where one set of actions is induced by the policy $\{\nu_{t'}\}_{t'=1}^{t-1}$ and the other set of actions is generated by the behavoir policy $\{\pi^b_{t'}\}_{t'=1}^{t}$. To distinguish two sources of actions, in the sequel, we use $A_t$ to represent the observed action (the action taken by the behavior agent) and $A^\nu_t$ to represent the action induced by $\nu$.  
For notation simplicity, in the following, we define $Z_t = \{O_{1:t}, A^\nu_{1:t-1}\} \in \mathcal{Z}_t$, where $\mathcal{Z}_t: = \mathcal{O}_{1:t} \times \calA_{1:t-1}$. To start with, we make the following assumptions.
\begin{assumption}
\label{ass:basic_fh}
    There exists a sequence of reward proxy variables $\{W_t\}_{t=1}^T$ such that 
\label{ass:RewardProxy} 
$W_t\indep A_t \mid (U_{1:t}, O_{1:t}, A_{1:t-1})$ and $W_t\dep U_{1:t} | (O_{1:t}, A_{1:t-1})$,  for $1\le t\le T$. 
\end{assumption}
Assumption \ref{ass:basic_fh} requires the existence of a sequence of proxy variables $\{W_t\}_{t=1}^T$, which are not affected by the previous action but correlated with the past hidden information. 
If one does not wish to assume an additional proxy variable, alternatively one may set $W_t$ as $O_t$ and take $Z_t = \{O_{1:t-1}, A^\nu_{1:t-1}\} \in \calZ_t$. Then Assumption \ref{ass:basic_fh} automatically holds under the current problem setting. However, in this case, the super policy $\nu_t$ needs to be defined over $\mathcal{O}_{1:t-1} \times \calA_{1:t-1} \times \calA_{1:t}$, which excludes the current observation $O_t$. In the following discussion, we  assume the existence of $\{W_t\}_{t=1}^T$. To identify $\calV(\nu)$ and ultimately $\nu^\ast$ under unmeasured confounding, we assume the existence of a class of $Q$-bridge functions below. Similar to the discussion in Section \ref{sec:identification_contextual_bandits}, the existence of such brdige functions can be satisfied under some completeness and regularity assumptions (see Assumptions S2 and S3  ).
\begin{assumption}
    \label{ass:q_existence}
  \revv{For any policy $\nu \in \Omega$,}  there exists a class of $Q$-bridge functions $\{\qnu_t\}_{t=1}^T$, where $\qnu_t$ is defined over $\calW \times \calZ_t \times \calA_{1:t} \times \calA$ for $t = 1,\dots, T-1$, such that for every $(u_{1:t}, o_{1:t}, a_{1:t}) \in \calU_{1:t} \times \mathcal{O}_{1:t} \times \calA_{1:t}$ and every $a_{1:t} \in \calA_{1:t}$, 
    {\small \begin{equation}
    \label{eqn:Q-bridge-1}
    \textstyle \EE^{\nu}\left[\sum_{t'=t}^T R_{t'} | U_{1:t}, O_{1:t}, A^\nu_{1:t-1}, A_{1:t} = a_{1:t}\right] = \EE\left[\sum_{a\in\calA} \qnu_t(W_t,Z_t, a_{1:t}, a)\nu_t(a\mid Z_t, a_{1:t}) |  U_{1:t}, O_{1:t}, A^\nu_{1:t-1} \right],
    \end{equation}
    }
where the Q-bridge function of the last step $q^\nu_T$ defined over over $\calW \times \calZ_T  \times \calA$ satisfies
    {\small \begin{equation}
        \label{eqn:Q-bridge-2}
        \textstyle \EE^{\nu}\left[R_T | U_{1:T}, O_{1:T}, A^\nu_{1:T-1}, A_{1:T} = a_{1:T}\right] = \EE\left[\sum_{a\in\calA} \qnu_T(W_T, Z_{t}, a)\nu_t(a\mid Z_T, a_{1:T}) |  U_{1:T}, O_{1:T}, A^\nu_{1:T-1} \right].
        \end{equation}
        }
\end{assumption}
\rev{
Similar to Section \ref{sec:identification_contextual_bandits}, we require the following basic assumption to hold in the setting of sequential decision making. 
\begin{assumption}
    \label{ass:consistency-long}
\rev{(a) Consistency: $R_t = R_t(A)$, $t=1,\dots, T$. (b) Positivity: and $P(A_t = a \mid O_{1:t}, A_{1:{t-1}}, O_o) > 0$ for any $a \in \calA$, $t=1,\dots, T$.}
\end{assumption}
}

Let $O_0$ denote some pre-collected observation before the decision process initiates. We impose the following additional assumption for $O_0$.
\begin{assumption}
    \label{ass:POMDP_more}
       $(R_t, W_{t+1}, W_t, O_{t+1}, U_{t+1}) \ind  O_0 | U_{1:t}, O_{1:t}, A_{1:t}$, for $1\leq t \leq T-1$.
\end{assumption}
The role of $O_0$ is similar to the role of the action proxy $Z$ in the setting of contexual bandits. Given all the history and current state variables and actions, we assume it to be independent of the current reward ($R_t$) and the current and future proxy variables ($W_t$ and $W_{t+1}$). In practice,  state variables at the next step ($U_{t+1}, O_{t+1}$) often solely depend on the history and current state variables and actions, leading to a natural independence between $O_0$ and $(U_{t+1}, O_{t+1})$.

Given aforementioned assumptions,  we have the following identification result.

\begin{theorem}\label{thm:POMDP2_existence2}
Suppose Assumptions \ref{ass:basic_fh}-\ref{ass:POMDP_more}, and certain completeness and regularity conditions (Assumptions S2 and S3) 
in {Section S3}  hold. 
Let $\qnu_{T+1}=0$.
At time $T$,
the Q-brige function $q^\nu_{T}(\cdot, \cdot, \cdot)$ is obtained via solving
\begin{align}
  \label{eq:POMDP-identification-T}
  \EE \left\{\qnu_T(W_T, (O_{1:T}, A_{1:T-1}),  A_T) - R_T  |   (O_{1:T}, A_{1:T-1}), O_0, A_T \right\} =0.
\end{align}
From $t= T-1, \dots, 1$, the Q-bridge function$q_t$ can be obtained via solving the following linear integral equation. In particular, for any $a_{1:t}  \in  {\calA}_{1:t}$,  $\qnu_t(\cdot, \cdot, a_{1:t}, \cdot)$   
is the solution to 
{\small 
\begin{multline}
  \label{eq:POMDP-identification2}
  \EE \left\{\qnu_t(W_t, (O_{1:t}, A_{1:t-1}), {a}_{1:t},  A_t) - R_t  - \right. \\
  \left. \sum_{a\in \calA}\qnu_{t+1}(W_{t+1},(O_{1:t+1}, A_{1:t}), ({a}_{1:t},  A_{t+1}), a )\nu_{t+1}(a\mid (O_{1:t+1}, A_{1:t}), ({a}_{1:t},    A_{t+1}))|   (O_{1:t}, A_{1:t-1}), O_0, A_t \right\} =0,  
\end{multline}
}
Then the policy value  of $\nu \in   \Omega$ can be identified as
\begin{align}\label{eqn: history dependent policy value 2}
\calV(\nu) = \EE\left[\sum_{a\in \calA}\qnu_1(W_1, O_1, A_1, a) \nu_1(a \mid   O_1, A_1)\right].
\end{align}
\end{theorem}
In Assumption \ref{ass:q_existence}, different from the definition of $q_t^\nu$ for $t<T$, $q^\nu_T$ does not depend on the actions from actions $a_{1:T}$ produced from the behavior policy. The intuitive reason behind it is that at the last step $T$, as long as all the state variables $U_{1:T}, O_{1:T}$ and actions $A_{1:T}^\nu$ produced by the policy $\nu$ are conditioned, the reward $R_T$ does not depend on the actions produces by the behavior policy. However, for $t<T$, the later rewards $R_{t+1}, \dots, R_T$ that produced by the policy $\nu$ will depend on the policies $\nu_{t+1}, \dots, \nu_T$, and therefore actions $a_{1:t}$ produced by the behavior policy, even though conditioning on $U_{1:t}, O_{1:t}$ and $A_{1:t}^\nu$.

\section{Estimations and Algorithms}
\label{sec:est}

In this section, we introduce our super-policy learning algorithms based on the identification results developed in Section \ref{sec:identification_contextual_bandits} and Section \ref{sec:identification_DTR},  and provide corresponding estimation procedures.

\subsection{Confounded Contextual Bandits: Algorithm Development}
\label{sec:est_contextual}
In Algorithm \ref{alg:contextual2}, we summarize the steps in learning the super policy using the observed data for contextual bandits. 
The first key step is to estimate the bridge function $q$ by the linear integral equation stated in Lemma \ref{lm: solve q contextual}.  The second key step is to estimate the projection term $\EE [\hat q(W,S,a) | S =s , Z= z, A = a']$ for every $a \in \calA$, using the estimated bridge function $\hat q$.

\begin{algorithm}[t]
\SetAlgoLined
\textbf{Input:} Data $\mathcal{D} = (S_i, Z_i, A_i, R_i, W_i)_{i=1}^n$. \\
Obtain the estimation of the bridge function $\hat q$ by solving the estimation equation
\rev{\eqref{eqn:qest-contextual}}
using data $\mathcal{D}$\\
Implement any supervised learning method for estimating $\EE\left[\hat q(W, S, a) | S, Z, A\right]$.\\
Compute $a^\ast = \argmax_{a \in \calA}\hat \EE\left[\hat q(W, S,a) | S = s, Z = z, A = a'\right] \quad \forall (s, z, a') \in \calS \times \calZ \times  \calA$.\\
\textbf{Output:} $\hat {\nu}^\ast$ with $\hat {\nu}^\ast(a^\ast | s, z, a') = 1$ and $\hat {\nu}^\ast(\tilde a | s, z, a') = 0$ for $\tilde a \neq a^\ast$.
\caption{Learning Algorithm for the contextual bandits under unmeasured confounding}\label{alg:contextual2}
\end{algorithm}

When $\calS \times \calZ \times \calA \times \calW$ are all finite and discrete, the bridge function and the projection term  can be straightforwardly estimated via empirical average. In the following, we 
focus on the case where the function approximation is needed. Specifically, we adopt the procedure for solving the conditional moment estimation procedure in \cite{dikkala2020minimax}, and propose to estimate 
$Q$-bridge function by 
    \begin{align}
    \label{eqn:qest-contextual}
        \hat{q} := \argmin_{q' \in \calQ} \left\{ \sup_{g\in \calG} \tilde{\Psi}(q', g) - \lambda\left(\|g\|^2_{\calG} + \frac{U}{\Delta^2} \|g\|_{2,n}^2 \right) \right\} + \lambda\mu \|q'\|_{\calQ}^2,
    \end{align}
    where $\tilde{\Psi}(q',g) = \frac{1}{n}\sum_{i=1}^n \left\{ q'(W_i, S_i, A_i) - R_i\right\}g(Z_i, S_i, A_i)$; $\|g\|_{2,n}^2 = n^{-1} \sum_{i=1}^n g^2(Z_i, S_i, A_i)$ for $g \in \calG$; $\calQ$ is the imposed model for $q$ (the solution of \eqref{eqn: linear integral equation contextual})
    ; $\calG$ is the function space where the test functions $g$ come from. In addition, $\lambda, \mu, \Delta, U>0$ are some tuning parameters. The motivation behind \eqref{eqn:qest-contextual} is  that 
when $\lambda$, $\lambda\mu \rightarrow 0$ and $\lambda U /\Delta \asymp 1$, the solution of the following population-version min-max optimization problem
\begin{align*}
    \argmin_{q' \in \calQ}  \sup_{g\in \calG} \EE\left\{[q'(W, S, A) - R] g(Z, S, A) - \frac{1}{2} g^2 (Z, S, A)\right\}
\end{align*}
is equivalent to the solution of 
following optimization problem 
\begin{align*}
    \argmin_{q' \in \calQ} \EE \left\{ \EE[q'(W, S, A) - R | Z, S, A ]^2\right\},
\end{align*}
when the space $\calG$ of testing functions is rich enough.

In practice, spaces $\calQ$ and $\calG$ are user-specified. To increase flexibility, $\calQ$ and $\calV$ can be implemented using growing linear sieves, reproducing kernel Hilbert spaces (RKHSs) and deep neural networks. When $\calQ$ and $\calG$ are taken as RKHSs, the optimization seems infinite-dimensional. However, due to the well-known representer theorem,  one can show that there exists a closed-form solution that lies in a finite-dimensional space. 
 For more information on deriving the closed-form solution, as well as guidance on hyper-parameter tuning and strategies to improve computational efficiency, we refer readers to Section E.3 of \cite{dikkala2020minimax}.

Meanwhile, the conditional moment framework can be  adopted to obtain the projection term. 
Here, we propose to perform the estimation via the empirical risk minimization: 
  \begin{align}
        \label{eqn:proj-contextual}
        \hat g(\cdot, \cdot, \cdot\ ;\ \hat q(\cdot, \cdot, a)) : = \argmin_{g \in \calG}\frac{1}{n}\sum_{i=1}^n \left[ g(S_i, Z_i, A_i) - \hat q(\cdot, \cdot, a)\right]^2 + \mu' \|g\|_{\calG}^2,
    \end{align}
    where $\hat q$ is defined in \eqref{eqn:qest-contextual} and $\mu'>0 $ is a tuning parameter. Similarly, one can take the pre-specified space $\calG$ in \eqref{eqn:proj-contextual} as growing linear sieves and RKHSs, which result in typical penalized spline regression and kernel ridge regression respectively.

\subsection{Confounded Sequential Decision Making:  Algorithm Development}
\label{sec:est_fh}
Given the identification results in Theorems \ref{thm:POMDP2_existence2} (or Theorem 
S1), to obtain the super-policy $\nu^\ast$, one 
solution is to directly search the optimal policy over the space of super-policies that maximize the estimated value, i.e.,
$$
    \hat \nu = \argmax_{\nu \in \Omega} \widehat{\calV}(\nu),
$$
where $\widehat{\calV}(\nu)$ is obtained by iteratively estimating $q_t^\nu$ through \eqref{eq:POMDP-identification2} (or $b^\nu_t$ through 
(S8)
) from $t = T$ to $t = 1$ with fixed $\nu$. 

However, when $T$ is large and models imposed for estimating bridge functions are complex (e.g., deep neural networks), direct optimizing $\widehat{\calV}(\nu)$ requires extensive computational power. 
Therefore, we restrict our focus to a special case of the sequential setting described in Section \ref{sec:fh1}, under which a more practical algorithm with theoretical guarantee can be derived. 
We leave the development of efficient algorithms under general settings as future work. Motivated by Theorem \ref{thm:POMDP2_existence2}, we propose a fitted-Q-iteration (Q-learning) type algorithm (Algorithm \ref{alg:fh2}) 
for practical implementation. In particular, Algorithm \ref{alg:fh2} has the theoretical guarantee (which we will discuss in {Section S6}
under the following memoryless assumption. 


\begin{assumption}[Memoryless Unmeasured Confounding]
\label{ass: memoryless}
For $2 \leq t \leq T$, 
$W_t \indep U_{1:(t-1)} | (O_{1:t}, A_{1:t})$,  $A_t \indep A_{1:t-1} | (U_{1:t},  O_{1:t})$. At the last step $T$, we additionally assume $(U_T, W_T) \indep A_{1:T-1} | (U_{1:T-1}, O_{1:T})$.
\end{assumption}
The memoryless assumption plays an important role in deriving the algorithm wherein policies are learned sequentially, starting from the last step and working backward. Similar conditions have been commonly imposed in the literature to handle unmeasured confounding in a sequential setting \citep{kallus2020confounding,Fu2022,shi2022off, xu2023instrumental}. Mainly, it ensures that the projection step guarantees the optimality under the distributions regarding to both the behavior policy and the induced policy.

\begin{algorithm}[h]
\caption{Super RL for the confounded POMDP}
\label{alg:fh2}
\SetAlgoLined
\textbf{Input:} Data $\mathcal{D} = \{\calD_t\}_{t=1}^{T-1}$ with $\calD_t = \{(O_{i,1:t}, A_{i,1:t}, R_{i,t}, W_{i,t}, O_{i,t+1}, A_{i,t+1}, W_{i,t+1})\}_{i=1}^n$. \\
Let $\hat q_{T+1}=0$ and $\hat \nu^\ast_T$ be an arbitrary policy.\\
At time $T$, obtain $\hat q_T$ for $q_T$ by solving \eqref{eq:POMDP-identification-T}. Compute $\hat \EE[\hat q_T(W_T,(O_{1:T}, a^\nu_{1:T-1}), A_T, a^\nu) | O_{1:T}=o_{1:T},  A_{1:T} = a_{1:T}]$ for $a^\nu\in \calA$ and $a^\nu_{1:T-1} \in \calA_{1:T-1}$ using any supervised learning method and obtain the estimated super policy $\hat \nu^\ast_T$ as 
$\hat {\nu}^\ast_T(a^\ast | o_{1:T},a^\nu_{1:T-1}, a_{1:T}) = \mathbbm{1}\left\{\argmax_{a \in \calA} \hat \EE[\hat q_T(W_T,(O_{1:T}, a^\nu_{1:T-1}), a) | O_{1:T}=o_{1:T},  A_{1:T} = a_{1:T}] \right\}$  for any $a^\nu_{1:T-1} \in \calA_{1:T-1}$.
\\
Repeat for $t=T-1,\dots,1$:\\
\Indp
Obtain an estimator $\hat q_t$ for $q_t$ 
  by solving \eqref{eq:POMDP-identification2} using data $\mathcal{D}_t$ and $\hat q_{t+1}$ obtained from the last iteration. \\  
Compute $\hat \EE[\hat q_t(W_t,(O_{1:t}, A_{1:t-1}), (a_{1:t-1}, A_t), a) | O_{1:t},  A_{1:t}]$ for $a\in \calA$ and $a_{1:t-1} \in \calA_{1:t-1}$ using any supervised learning method and obtain the estimated super policy $\hat \nu^\ast_t$ as 
$\hat {\nu}^\ast_t(a^\ast | o_{1:t},a^\nu_{1:t-1}, (a_{1:t-1}, a_t)) = \mathbbm{1}\left\{\argmax_{a \in \calA} \hat \EE[\hat q_t(W_t,(O_{1:t}, A_{1:t-1}), (a_{1:t-1}, A_t), a) | O_{1:t}=o_{1:t},  A_{1:t-1} = a^\nu_{1:t-1}, A_t = a_t] \right\}$ for any $a_{1:t-1} \in \calA_{1:t-1}$.\\
\Indm
\textbf{Output:} $\hat \nu^\ast = \{\hat {\nu}^\ast_t\}_{t=1}^T$.
\end{algorithm}



In Algorithm \ref{alg:fh2}, the iteration is conducted from the final time step $t=T$ to the first time step $t=1$. At each iteration $t$, there are two main steps. One is to estimate the  $Q$-bridge function $q_t$ and the other is to perform the projection. 
We take similar procedures as described in Section \ref{sec:est_contextual} to perform these two steps. 
In the step of estimating the $Q$-bridge function, we follow the construction in \cite{dikkala2020minimax} to derive the estimators. We construct the objective function for the $Q$-bridge function at the last step $T$ based on \eqref{eq:POMDP-identification-T}. 
For the steps $t <T$ and for every combination of $a_{1:t}$,  we construct the objective function based on \eqref{eq:POMDP-identification2} to learn the $Q$-bridge function $q^\nu_t$, where we replace $\nu_{t+1}$ and $q^\nu_{t+1}$ with the estimated ones ($\hat{\nu}_{t+1}$ and $\hat{q}^\nu_{t+1}$ ) obtained from the previous step.  
For the projection step, it performs differently at step $T$ and steps $t< T$ as shown in Line 3 and Line 6 in Algorithm \ref{alg:fh2}. 
In particular, at the last step $T$, the dependence of the policy on actions produced by the behavior agents ($a_{1:T}$) is done by conditioning the last step $Q$-bridge function on the observed actions $A_{1:T}$. At the steps $t<T$, the dependence of the policy on $a_{1:t-1}$ is directly through the input $a_{1:t-1}$ of the $Q$-bridge function $q^\nu_t$; the dependence on $a_t$ is through conditioning on the observed action $A_t$.  
A major reason for such difference in the projection steps is that the Q bridge function at the last step does not depend on the actions produced by the behavior actions $a_{1:T}$. The conditioning set $A_{1:T}$ in the projection step plays the role as the actions produced by the behavior policy,  and through doing this we could learn the policy at the last step that depends  not only on the previous taken actions but also on the actions produced by the behavior policy as well.  Same implementation procedures as discussed in Section \ref{sec:est_contextual} can be used. 
In {Section S8.1} 
, we list implementation details for these two steps.

\section{Super-policy Learning with Regret Guarantees}
\label{sec:regret}
In this section,  we establish the  finite-sample regret bounds for algorithms developed in Section \ref{sec:est}. In particular, we focus on deriving the finite-sample upper bound for the regret of finding the super-policy in both contextual and sequential settings. The regret of any generic policy $\widetilde \nu$ is defined as
\begin{align}\label{def: regret}
    \text{Regret}(\widetilde \nu) \equiv \calV(\nu^\ast) - \calV(\widetilde \nu).
\end{align}

Due to space constraints, we present the contextual bandits results in the main paper.  The regret bounds for the sequential setting are provided in {Section S6}. 
Specifically, we derive the regret bound for Algorithm \ref{alg:fh2} under the memoryless setting discussed in Section \ref{sec:est_fh}.

Let $\hat{\nu}^\ast$ denote the output of Algorithm \ref{alg:contextual2} 
which relies on the estimation of the bridge function $q$ given by \eqref{eqn: linear integral equation contextual}. Define the $\mathcal{L}_2$ norm of a generic function $f$ as $\|f\|_2 \equiv \sqrt{\EE [f^2]}$. 
Let $g(S, Z, A\ ; f) \equiv \EE[f (W,S) \mid S, Z, A]$ for any $f$ defined over $\calW \times \calS$. For a given projection estimator $\hat\EE$,  
let $\hat g(S, Z, A\ ; f) \equiv \hat\EE[f (W,S) \mid S, Z, A]$ denote the corresponding estimator. Let
 \begin{align*}
   p_{\max} = \sup_{u,s,z,a', \nu \in \Omega} \frac{\sum_{a\in \calA} \pi_b(A=a\mid U=u, S=s)\nu(A'=a'\mid Z=z,S=s,A=a)}{\pi_b(A'=a'\mid U=u, S=s)}.
\end{align*}
Define the projection error as
$\xi_n : =  \sup_{q \in \calQ, a\in \calA}   \left\| g[\cdot, \cdot, \cdot\ ; \ q(\cdot, \cdot, a) ] - \hat g[\cdot, \cdot, \cdot\ ;\  q(\cdot, \cdot, a) ]\right\|_{2} , $
and the bridge function estimation error as
$\zeta_n : = \left\| q - \hat q\right\|_{2}.$
The following Lemma shows that the regret bound can be controlled through bounding the $Q$ function estimation error and the projection estimation error. 
\begin{lemma}
\label{thm:regre_contextual}
Suppose $q$ belongs to the function class $\calQ\subset \calW \times \calS \times \calA$. 
Then we obtain the following regret decomposition
$$\text{Regret}(\hat{\nu}^\ast) \leq 2(\xi_n + p_{\max}\zeta_n).$$
\end{lemma}
Lemma \ref{thm:regre_contextual} indicates that the error bound consists of two compoennts: the estimation error for the bridge function and the estimation error for the projection step.  Suppose $\hat{q}$ and the projection estimator are computed by the estimation procedures described in Section \ref{sec:est_contextual}. 
When $\calQ$ (the function space for $q$) and $\calG$ (the function space for test functions and the function space for the projected function) are VC-subgraph classes, we have the following theorem for the regret guarantee. Results when $\calG$ and $\calQ$ are reproducing kernel Hilbert spaces (RKHSs) are provided in  {Section S8.4}. 

\begin{theorem}
    \label{cor:regret-contextual}
    If the star-shaped spaces $\calG$ and $\calQ$ are VC-subgraph classes with VC dimensions $\mathbb{V}(\calG)$,  and $\mathbb{V}(\calQ)$ respectively. Under  assumptions in Theorems 
    S4 and S7 in the supplementary material, 
    with probability at least $1-\delta$,  
     \begin{align*}
       \text{Regret}(\hat v^\ast)  \lesssim n^{-1/2}p_{\max}  \sqrt{\log(1/\delta) + \max\left\{ \mathbb{V}(\calG),  \mathbb{V}(\calQ)\right\}},
    \end{align*}
    where for any two positive sequences $\{a_n\}_n$, $\{b_n\}_n$, $a_n\lesssim b_n$ means that there exists some universal constant $C>0$ such that $a_n\le C b_n$ for any $n$. 
\end{theorem}
Theorem \ref{cor:regret-contextual} provides the finite-sample regret bound for the super-policy learning algorithm  under the setting of confounded contextual bandits. The bound is determind by the sample size $n$, the overlap quantity $p_{\max}$ and function spaces \(\mathcal{Q}\) and \(\mathcal{G}\). Suppose $p_{\max}$ is bounded by a constant and the VC dimensions are $K$, then the derived regret bound achieves the rate $\sqrt{K/n}\log n$. \rev{In \cite{athey2021policy}, they derived  the regret bound of the common policy assuming there is no unmeasured confounders. Our regret bound matches their upper bound ($\sqrt{1/n}$) regarding the sample size up to a $\log$ order. }
\rev{In practice, there are different choices of VC-subgraph classes. For instance, if one assumes the linear model of $\calQ$, i.e,  $\calQ = \{q: q(s,a) = \sum_{j=1}^d \alpha_j b_j(s,a), \alpha_j \in \mathbb{R}, j =1,\dots, d \}$, where $b_j(\cdot, \cdot), j=1,\dots, d$ are pre-specified basis functions. Then we have $\mathbb{V}(\calQ)\leq d+1$. }

{\section{Simulations} \label{sec:simu}}
{\subsection{Simulation Study for Contextual Bandits}\label{sec:simu_contextual}}

{In this section, we conduct two simulation studies to evaluate the performance of the proposed super-policy. The first one is a contextual bandit example with discrete feature values. We aim to demonstrate the super-policy performs better when the behavior policy reveals more information about the unmeasured confounders. The second one is a contextual bandit example with a continuous state space. It is used to demonstrate the performance of our algorithm using the bridge function.}

\medskip

\noindent {\textbf{A contextual bandit example with discrete feature values:}}
{ Similar to the toy example described in Section \ref{sec:contextual}, we take $S$ and $U$ as independent binary variables such that $\Pr(S=1) = 0.5$ and $\Pr(U=1) = 0.5$. The binary action $A$ is generated by the following conditional probabilities
\[
   \Pr(A = 1 \mid U = 0) = \epsilon,  \qquad 
   \Pr(A = 1 \mid U = 1) = 1- \epsilon,
\]
with different choices of $\epsilon \in [0,1]$. The larger the $|\epsilon - 0.5|$ is, the more information of $U$ is revealed in the observed action $A$.  Both the reward proxy $W$ and the action proxy $Z$ are binary and are generated according to the following conditional probabilities
\[
   \Pr(W = 1 \mid U = 0) = 0.4,  \qquad 
   \Pr(W = 1 \mid U = 1) = 0.6; 
\]
\[
   \Pr(Z = 1 \mid U = 0) = 0.4,  \qquad 
   \Pr(Z = 1 \mid U = 1) = 0.6.  
\]
Moreover, $W$ and $Z$ are conditionally independent given $U$.  The observed reward is computed by $R = (U- 0.5)  (A-0.5) + \epsilon$ where $\epsilon \sim N(0, 0.5)$.}

{Three types of policy classes are considered. 
\begin{enumerate}
    \item \sonly{}: $\mathcal{S} \rightarrow \mathcal{P}(\mathcal{A})$. The policy only depends on the observed state $S$. 
     \item \szonly{}: $\mathcal{S} \times \mathcal{Z} \rightarrow \mathcal{P}(\mathcal{A})$. The policy depends on on the observed state $S$ and the action proxy $Z$. 
      \item \super{}: $\mathcal{S} \times \mathcal{Z} \times \mathcal{A} \rightarrow \mathcal{P}(\mathcal{A})$. The super-policy class where the policy depends on the observed state $S$, the action proxy $Z_t$, and observed action $A$. 
\end{enumerate}
We implement Algorithm \ref{alg:contextual2} to estimate the corresponding optimal policies for different policy classes. Note that for $\sonly{}$ and $\szonly{}$, we perform the projection step (line 4) by conditioning on $S$ and $(S,Z)$ respectively.  Since the feauture values are discrete, we use the empirical averages to approximate all the conditional expectations. In this simulation study, we consider the sample size $n = 5000$.
As Table \ref{tab:tabular_contextual} shows, the super-policy produces smallest regret as $\epsilon$ deviates from $0.5$ more, while the estimated optimal policies such as $\sonly{}$ and $\szonly{}$ do not change and have larger regrets. }

\begin{table}[t]
\caption{{Simulation results for  the discrete feature values setting described in \ref{sec:simu_contextual} under different choices of $\epsilon$. We replicate the simulation for $50$ times. Mean regret values for estimated optimal policies under different policy classes are provided (and a smaller regret value indicates a better performance).  Values in the parentheses are the standard deviations of the regret values. }}
\label{tab:tabular_contextual}
\begin{center}
\begin{tabular}{| l| lll|}
\hline
&\bf \sonly{}   &\bf \szonly{}   &\bf \super{} \\ \hline 
$\epsilon = 0.5$ &  0.25 (3.1e-04)& \textbf{0.21} (1.7e-02)& \textbf{0.21} (1.4e-02) \\ \hline
$\epsilon = 0.7$ &  0.25 (3.1e-04)& 0.22 (1.8e-02)& \textbf{0.18} (3.5e-02) \\ \hline
$\epsilon = 0.9$ &  0.25 (2.5e-04)& 0.24 (1.2e-02)&  \textbf{0.17} (8.6e-02) \\ \hline
\end{tabular}
\end{center}
\end{table}

\medskip
\noindent {\textbf{A contextual bandit with a continuous state:}}
{ In this setting, we take $S$ and $U$ as independent Gaussian random variables such that $S \sim N(0,1)$ and $U \sim N(0,1)$. The binary action $A$ is  generated by the following conditional probabilities
\[
   \Pr(A = 1 \mid U > 0) = \epsilon,  \qquad 
   \Pr(A = 1 \mid U \leq 0 ) = 1- \epsilon,
\]
with different choices of $\epsilon \in [0,1]$. Again, the larger the $|\epsilon - 0.5|$ is, the more information of $U$ is revealed in the observed action $A$.  We generate $W$ and $Z$ according to the following conditional probabilities
\[
   W\mid (S, U) \sim N( S + 3U,  1);
\]
\[
   Z\mid (S, U)  \sim N( 3S + U,  1).
\]
Moreover, $W$ and $Z$ are conditionally independent given $(S, U)$.  The observed reward is computed by $R = U  (A-0.5) + \epsilon$ where $\epsilon \sim N(0, 0.5)$.
For this continuous setting, we compute the $Q$-bridge function via the min-max conditional moment estimation described in Section \ref{sec:est_contextual} by taking $\mathcal{G}$, $\mathcal{Q}$ as reproducing kernel Hilbert Spaces (RKHSs) equipped with Gaussian kernels. The bandwidths of Gaussian kernels are selected by the median heuristic.  Tuning parameters of the penalties are selected by cross-validation. Computation details can be found in Section E of \cite{dikkala2020minimax}. As for the projection step, we adopt 
the linear regression to perform the estimation. 
In this simulation study, we take the sample size $n = 1000$.} 

{Table \ref{tab:continous_contextual} shows the simulation results over $50$ replications. The observation is consistent with that in the discrete feature values setting.  The super-policy clearly outperforms the other two commonly used optimal policies when $\epsilon$ deviates from $0.5$.}

\begin{table}[t]
\caption{{Simulation results for  the continuous setting described in \ref{sec:simu_contextual} under different choices of $\epsilon$. The simulation is performed over $50$ simulated datasets. Mean regret values for estimated optimal policies using different policy classes are provided. Smaller regret values indicate better performance.  Values in the parentheses are the standard deviations.}} 
\label{tab:continous_contextual}
\vspace{-0.3cm}
\begin{center}
\begin{tabular}{|l | lll|}
\hline
&\bf \sonly{}   &\bf \szonly{}   &\bf \super{} \\ \hline 
$\epsilon = 0.5$ &  0.40 (2.32e-03)& \textbf{0.11} (1.80e-03)& \textbf{0.11} (1.78e-03) \\ \hline
$\epsilon = 0.7$ &  0.40 (2.35e-03)& 0.12 (2.08e-03)& \textbf{0.10} (2.67e-03) \\ \hline
$\epsilon = 0.9$ &  0.40 (2.03e-03)& 0.12 (5.29e-02) & \textbf{0.06} (6.32e-03) \\ \hline
\end{tabular}
\end{center}
\vspace{-0.6cm}
\end{table}

{\subsection{A Simulation Study for Sequential Decision Making}\label{sec:simu_fh}}

In this section, we perform a simulation study to evaluate the performance of the super-policy in the sequential decision making. Specifically, we follow the data generation process described in Section F.1 of \cite{Miao2022}. Mainly, our $O_t$ corresponds to their $S_t$ for $t = 1,\dots, T$ and $O_0$ corresponds to their $Z_1$.  Other variables match exactly with their notations,  and we only change the reward function to $R_t = \mathrm{expit}\{U_t(A_t - 0.5)\} + e_t$, where $e_t\sim \mathrm{Uniform}[-0.1,0.1]$ and $ \mathrm{expit}(x) = 1/(1+\exp(-x))$. We take the sample size as $n = 2000$ and the length of episode $T = 2$.  Note that this setting satisfies the memoryless assumption (i.e., Assumption \ref{ass: memoryless}). We implement Algorithm \ref{alg:fh2} to estimate the optimal super policy (\super{}), and compare it with the common policy (\common{}) where the policy depends on observations $O_{1:t}$ and history actions $A_{1:t-1}$.   
We again use the RKHS to perform the min-max conditional moment estimation for obtaining a sequence of $Q$-bridge functions  and implement a linear regression to estimate the projections at every iteration. See implementation details in the discussion of the continuous setting in Section \ref{sec:simu_contextual}. To obtain the regret value, we estimate the optimal policy which depends on unobserved state variables $U_{1:t}$ and observed state variables $O_{1:t}$,  and use it to approximate  the oracle optimal value. \rev{In our current sequential decision making framework, we require $T$ to be relatively small as the policy and $Q$-bridge functions depends on all the historical observations and actions produced by both behavior policy and super policy.  This makes computation very challenging when $T$ becomes larger.  To address this challenge and enable practical implementation of the super policy in long-horizon sequential decision-making scenarios, by leveraging the episodic POMDP structure in \cite{rui2022off}, we simplified the $Q$ functiona and let the super policy depend on only the current observation $O_t$, action proxy $Z_t$ and action produced by the behavior policy $A_t$. Following the identification  and estimation framework detailed in {Section S11}, 
we conduct the simulation for $T=10$ and $T=20$ with $n=2000$ under the same data generating process described above. Detailed description and additional simulation results can be found in {Section S11.4}. }
Table \ref{tab:fh} summarises the simulation results over $50$ simulated datasets. As we can see, the super policy performs  significantly better than the common policy.

\begin{table}[t]
\caption{{Simulation results for the sequential decision making problem described in \ref{sec:simu_fh}. 
Mean regret values over $50$ simulated datasets  for estimated optimal policies under different policy classes are provided. The smaller regret values indicate better performances.  Values in the parentheses are the standard deviations of the regret values. }}
\label{tab:fh}
\vspace{-0.3cm}
\begin{center}
\begin{tabular}{ |l|ll|}
\hline
 & \bf \common{}    &  \bf \super{} \\ \hline 
 $T=2$ & 9.65e-02 (3.33e-03)&  \textbf{7.91e-02} (3.77e-03) \\
 \hline 
 $T=10$ &  1.73 (2.61e-01)& \textbf{0.697} (1.51e-01) \\
 \hline 
 $T=20$ & 5.68 (3.27e-01)& \textbf{3.54} (2.81e-01) \\
  \hline
\end{tabular}
\vspace{-0.6cm}
\end{center}
\end{table}

\medskip

\section{Real Data Applications}

{In this section, we evaluate the performance of our method on the dataset from a cohort study of patients with deteriorating health
who were referred for assessment for intensive care unit (ICU) admission in 48 UK National Health Service
(NHS) hospitals in 2010-2011 \citep{harris2015delay}. The data can be obtained from \cite{keele2020stronger}.  Our goal is to find an optimal policy on whether recommending the patients for admission that maximizes 7-day survival rates.}

{This application corresponds to the contextual bandits problem. In the dataset, there are 13011 patients, of whom 4934 were recommend to be admitted to ICU by doctors ($A=1$) and the remaining were not ($A=0$). If a patient survived or censored at day 7, we let the response $Y=100$, otherwise, we take the response as $Y = 0$.
We include patients' age, sex, and sequential organ failure assessment score (sofa\_score) as baseline covariates. Usually the number of open beds in ICU may limit the real 
admissions of patiens and therefore affect the survival of patients, we also include the number of open beds in ICU in the baseline covariates. 
For the remaining measurements, 
following the idea in Section 6.1 in \cite{tchetgen2020introduction} for selecting proxy variables, we look at variables that are strongly correlated with the treatment and the outcome. As a results, we take the National
Health Service national early warning score (news\_score) as the action proxy $Z$ and the indicator of periarrest as the reward proxy $W$.
}

{We compare the super-policy with the two common policies \sonly{} \szonly{} described in Section \ref{sec:simu_contextual}. To make it more comparable, we use the same estimating procedure for the bridge functions considered in these three methods.  In addition, the RKHS modeling for the min-max conditional moment estimation is taken to obtain the $Q$-bridge function. See details of the RKHS modeling  in the continuous setting in Section \ref{sec:simu_contextual}. \revv{
In the implementation, we utilize existing code from the repository \url{https://github.com/rui-miao/ProxITR} to obtain the estimated $Q$-bridge functions. Specifically, we take the function \texttt{ApproxRKHSIVCV} with the argument \texttt{n\_components = 15}, while keeping all other input arguments set to their default values.
}   We use the  linear regression to obtain the projection (line 4) in Algorithm \ref{alg:contextual2}.
} 

{To evaluate the value by different policies, we randomly separate $40\%$ of the data and use it as the evaluation set $\mathcal{E}$. More specifically, after obtaining the estimated optimal policies using $60\%$ of the data, we perform the policy evaluation of these three estimated optimal policies using the remaining $40\%$ of the data. Take $\hat q$ as the estimated bridge function using whole data. The evaluation is conducted as follows. 
$\calV(\nu) = \hat \EE \{\sum_{a \in \calA} \hat q(W, S , a) \nu(a \mid S, Z, A)\}$, for $ \nu \in \super{}$; $\calV(\pi) = \hat \EE \{\sum_{a \in \calA} \hat q(W, S , a) \pi(a \mid S, Z)\}$, for $ \pi \in \szonly{}$; $\calV(\pi) = \hat \EE \{\sum_{a \in \calA} \hat q(W, S , a) \nu(a \mid S)\}$, for $ \pi \in \sonly{}$, where the expectation $\hat \EE$ refers to the average with respect to the evaluation set $\mathcal{E}$.}
{Table \ref{tab:rhc} shows the evaluation results over $20$ random splits. As we can see, the super-policy produces higher evaluated policy values compared with the other two methods. }
\begin{table}[h]
\caption{{Evaluation results of the optimal policies learned from three different policy classes using ICU admission data. The averages of evaluation values over 20 random splits are presented. 
Larger values indicate better performances. Values in the parentheses are standard errors. }}
\label{tab:rhc}
\vspace{-0.3cm}
\begin{center}
\begin{tabular}{|lll|}
\hline 
\bf \sonly{}   &\bf \szonly{}   &\bf \super{} \\ \hline 
  88.18 (0.351)&  88.10 (0.277)& \textbf{88.70}  (0.266)\\ \hline
\end{tabular}
\vspace{-0.6cm}
\end{center}
\end{table}

\revv{Besides the application to ICU admission data, we also use the Multi-parameter Intelligent Monitoring in Intensive Care (MIMIC-III) dataset (\url{https://physionet.org/content/ mimiciii/1.4/}) as a sequential example to demonstrate the performance of estimated optimal policies from two policy classes (\common{} and \super{}). 
We consider two different horizons $T=2$ and $T=10$.  As the results show, two policies show similar performance in both settings, 
 which implies that the behavior actions may not contain additional information of unobserved state variables. As a futher investigation, we performed a random forest classification model to check if expert's  actions can be fully explained by the observed state variables. The accuracy on the test data set is over 99\%, which further supports the claim that the behavior actions do not contain additional information of unobserved state variables. 
More details can be found in {Section S9}. }

\section{Conclusion}
In this paper, we propose a super policy learning framework using offline data  for confounded environments. With the hope that actions taken by past agents may contain valuable insights into undisclosed information, we include the actions produced by the behavior agent as input for the decison making so as to achieve a better oracle. Built upon the idea of the proximal causal inference, we  develop several novel identification results for super policies under different settings including contextual bandits and sequential decision making. In particular, for the sequential decision making, we provide two distinct identification results, using either the $Q$-bridge functions and $V$-bridge functions, respectively.  Based on these results, we then introduce new policy learning algorithms for estimating the super policy, and we conduct an analysis of finite-sample regret bounds for these algorithms.   A series  of numerical experiments show the appealing performance of our proposed framework and  highlight its superiority over common policies that only rely on observed features.

We list several directions for the future work. First, an efficient algorithm for sequential policy learning using $V$-brdige functions will be particularly useful. Comparing to the approach using $Q$-bridge functions,  it does not require the reward proxy variables $W_t$. Second, it is of great interest to extend the idea of super policy learning to other identification frameworks with unmeasured confounders. Currently, we borrow the idea form proximal causal inference to establish our identification results. Some extensions could be investigated by using the instrumental variables \citep{Fu2022} or mediators \citep{shi2022off}.

\section*{Acknowledgments}
 The authors thank the reviewers for their helpful comments and suggestions. Portions of this research were
conducted with the computing resources provided by
UTD High Performance Computing (HPC).  The work of Jiayi Wang is partly supported by the National Science Foundation (DMS-2401272) and Texas Artificial Intelligence Research Institute (TAIRI). 
The authors report there are no competing interests to declare.

\section*{Supplementary Material}
The supplementary material, including additional theoretical results, proofs, and implementation details, is available for download at \url{https://www.tandfonline.com/doi/suppl/10.1080/01621459.2025.2574706?scroll=top}.



\bibliographystyle{chicago}
\spacingset{0.85}
\bibliography{reference}
\end{document}